\documentclass[conference]{IEEEtran}
\IEEEoverridecommandlockouts
\def\BibTeX{{\rm B\kern-.05em{\sc i\kern-.025em b}\kern-.08em
    T\kern-.1667em\lower.7ex\hbox{E}\kern-.125emX}}
\usepackage{cite}
\usepackage{amsmath,amssymb,amsfonts}
\usepackage{graphicx}
\usepackage{textcomp}
\usepackage{xcolor}
\usepackage{hyperref}
\usepackage{multirow}
\usepackage{url}
\usepackage{algorithm}
\usepackage{algpseudocode}
\usepackage{amsmath,amsfonts,bm}

\newtheorem{theorem}{Theorem}
\newtheorem{proof}{Proof}

\usepackage{acronym}
\newacro{MTL}{Multi-task learning}
\newacro{EEG}{Electroencephalogram}

\begin{document}
\title{Combining datasets to improve model fitting}
\author{\IEEEauthorblockN{Thu Nguyen}
\IEEEauthorblockA{\textit{Dept. of Holistic Systems} \\
\textit{Simula Metropolitan}\\
Oslo, Norway \\
thu@simula.no}
\and
\IEEEauthorblockN{Rabindra Khadka$^{\star}$}
\IEEEauthorblockA{\textit{Dept. of Computer Science} \\
\textit{Oslo Metropolitan University}\\
Oslo, Norway\\
radhikak@oslomet.no}
\and
\IEEEauthorblockN{Nhan Phan $^{\star}$}
\IEEEauthorblockA{\textit{Dept. of Mathematics and Computer Science} \\
\textit{University of Science}\\
\textit{Vietnam National University in Ho Chi Minh City}\\
Ho Chi Minh city, Vietnam \\
nhanmath97@gmail.com}
\and
\IEEEauthorblockN{Anis Yazidi}
\IEEEauthorblockA{\textit{Dept. of Computer Science} \\
\textit{Oslo Metropolitan University}\\
Oslo, Norway \\
Anisy@oslomet.no}
\and
\IEEEauthorblockN{Pål Halvorsen}
\IEEEauthorblockA{\textit{Dept. of Holistic Systems} \\
\textit{Simula Metropolitan}\\
Oslo, Norway \\
paalh@simula.no}
\and
\IEEEauthorblockN{Michael A. Riegler}
\IEEEauthorblockA{\textit{Dept. of Holistic Systems} \\
\textit{Simula Metropolitan}\\
Oslo, Norway \\
michael@simula.no}
}

\maketitle

\def \thefootnote{$\star$} \footnotetext{denotes equal contribution}
\begin{abstract}  
For many use cases, combining information from different datasets can be of interest to improve a machine learning model's performance, especially when the number of samples from at least one of the datasets is small. An additional challenge in such cases is that the features from these datasets are not identical, even though there are some commonly shared features among the datasets. To tackle this, we propose a novel framework called \textit{Combine datasets based on Imputation} (ComImp). In addition, we propose PCA-ComImp, a variant of ComImp that utilizes Principle Component Analysis (PCA), where dimension reduction is conducted before combining datasets. This is useful when the datasets have a large number of features that are not shared across them. Furthermore, our framework can also be utilized for data preprocessing by imputing missing data, i.e., filling in the missing entries while combining different datasets. To illustrate the performance and practicability of the proposed methods and their potential usages, we conduct experiments for various tasks (regression, classification) and for different data types (tabular data, time series data) when the datasets to be combined have missing data. We also investigate how the devised methods can be used with transfer learning to provide even further model training improvement. Our results indicate that can provide extra improvement when being used in combination with transfer learning. 
\end{abstract}
\begin{IEEEkeywords}
combining datasets, transfer learning, missing data, imputation
\end{IEEEkeywords}

\section{Introduction and Motivation}\label{sec-intro}

Machine learning algorithms usually improve as we collect more data. However, there are times when the data is scarce, such as in medicine. It is possible to combine different data sets in order to obtain more information. {In databases, a multi-table query is often used to combine data from multiple tables. However, it requires the tables need to have the same columns. In practice, however, datasets may have overlapping features but may not have identical columns.}  
In this work, we present two novel methods that efficiently combine datasets that have overlapping features 
to solve similar or new tasks. In addition, we are also investigating the impact of our method on transfer learning.

As a motivating example, suppose that we have two datasets
\begin{align*}\label{eq1}
\small
    \mathcal{D}_1 &= \begin{pmatrix}
    \text{person} &\text{height (cm)} & \text{weight (kg)} & \text{BSL (mg/dl)} \\
    1 & 120 & 80 & 80\\
    2 & 150 & 70 & 90\\
    3 & 140 & 80 & 85\\
    4 & 135 & 85 & 95
    \end{pmatrix},\\
    \mathcal{D}_2 &= \begin{pmatrix}
    \text{person} & \text{weight (kg)} & \text{calo/meal} &  \text{BSL (mg/dl)}\\
    5 & 90 & 100  & 100\\
    6 & 85 & 150 & 95\\
    7 & 92 & 170 & 82
    \end{pmatrix}.
\end{align*}
Here, the first two columns in $\mathcal{D}_1$ (height (cm), weight (kg)) form the input $ \mathcal{X}_1$, and the last column (BSL(mg/dl)) is the label $\mathbf{y}_1$. Similarly, the first two columns in $\mathcal{D}_2$ (weight (kg), calo/meal) form the input $ \mathcal{X}_2$, and the last column (BSL (mg/dl)) is the label $\mathbf{y}_2$. 
It is worth noticing that the number of samples from $\mathcal{D}_1$ and $\mathcal{D}_2$ are different and are collected from different individuals. Specifically, $\mathcal{D}_1$ is collected from individuals 1, 2, 3, 4, while $\mathcal{D}_2$ is collected from individuals 5, 6, 7. Then, combining the information from these datasets together can help increase the sample sizes and may allow more efficient model training, prediction, and inferences. However, $\mathcal{X}_1$ and $\mathcal{X}_2$ do not have the features height and calo/meal in common, which creates a challenge in combining the two datasets together. 

This motivates us to propose \textbf{\textit{ComImp}} (Combine datasets based on Imputation), a framework that allows vertically combine datasets based on missing data imputation.
The algorithm works by vertically stacking the datasets by overlapping features (weight) and inserting columns of missing entries to a \textit{\textbf{component dataset}} (the datasets that are supposed to be merged) if the corresponding features do not exist in the dataset but exist in some other component dataset. Finally, the algorithm conducts imputation on the stacked data.

Moreover, note that $\mathcal{D}_1, \mathcal{D}_2$ have two features that are not in common. However, in many scenarios, the number of features not in common can be up to thousands. Then, simply using ComImp to combine datasets can result in a lot of noise and bias. Therefore, in addition to ComImp, we also propose PCA-ComImp, which is a variant of ComImp that incorporate dimension reduction as a prior step to reduce the number of entries to be imputed.


Further, it is common for data to have missing values due to instrument malfunctioning, skipped questionnaires by participants, or data being lost at different stages of the experiment. Since our methods use imputation to combine datasets, it can also act as a preprocessing step by filling in the missing entries and creating a complete dataset. In addition, our method can also be used with data augmentation techniques to increase the sample size and improve learning even further.

In summary, our contributions are as follows: (i) We propose a novel yet simple framework, named ComImp, for combining datasets to improve model fitting; (ii) We propose PCA-ComImp, a variant of ComImp, for combining datasets with large numbers of features not in common; (iii)  We analyze the effects of imputation methods on the proposed frameworks; (iv) We illustrate the advantage of the given methods when acting as preprocessing steps (imputation) when the datasets to be combined have missing values; (v) We demonstrate how ComImp can be used with transfer learning to provide even better classification accuracy; and (vi)  We provide various experiments to show the power of the ComImp and PCA-ComImp.

The rest of this paper is organized as follows: in section \ref{sec-related}, we review some related works. Next, in sections \ref{sec-method} and \ref{sec-pca-comimp}, we present the ComImp and PCA-ComImp algorithms for combining datasets, respectively. Then, in section \ref{sec:choice}, we analyze some aspects of choosing an imputer for ComImp and PCA-ComImp. Moreover, we illustrate the effectiveness and potential use of the proposed methods in section \ref{sec-exper}. Lastly, we conclude the paper and provide potential future works in section \ref{sec-concl}. 

\section{Related Work}\label{sec-related}
A lot of research has been done on utilizing information from related datasets to improve model fitting. These approaches fall under  the fields of \ac{MTL}, transfer learning, or dataset merging within a single data modality. 

\ac{MTL} seeks to learn multiple tasks together by leveraging shared information among tasks. The tasks generally have a high degree of similarities in \ac{MTL}. This helps address the issue of the limited availability of labeled data. The scarcely available labeled datasets, even if collected for different tasks, can be aggregated for multi-task learning purposes. One of the reasons that \ac{MTL} performs better than the single-task learning paradigm is the effective learning of more general representation from the datasets.
A \ac{MTL} model learns useful representations so that it is able to perform well for a new task that is closely related to the previously learned task. The \ac{MTL} models are broadly categorized as feature-based approaches and parameter-based approaches. The feature-based approach is based on the idea of sharing information. It seeks to learn common representation features among all the tasks \cite{misra2016cross}. The parameter-based approach is motivated by the notion of linking among tasks by placing priors on model parameters. Various methods have been proposed while sharing the parameters, such as task clustering method \cite{kumar2012learning}, 
and decomposition method \cite{zweig2013hierarchical}. 


Another way to make use of information from related datasets is to perform transfer learning. The idea of transfer learning is to transfer the weights learned from a pre-trained model and solve another task without training a model from scratch. Transfer learning is particularly useful in scenarios with a small training dataset \cite{TLsurvey,
cao2010transfer}. Transfer learning has achieved state-of-the-art performance in domains like natural language processing \cite{zoph2016transfer,ruder2019transfer} and natural images \cite{hussain2018study}. For example, the study by \cite{chen2019transfer} adopts transfer learning to use a large number of structurally incomplete data samples to improve the accuracy of fault diagnosis. Next, the work by \cite{deng2013sparse} applies transfer learning for emotion recognition from speech where features are transferred from multiple labeled speech sources. Another work by \cite{zhou2020imputing} proposes a novel method to impute missing gene expression data from DNA methylation data through a transfer learning–based neural network.  {On the other hand, for skin analysis, TATL \cite{nguyen2022tatl} is a segmentation framework that detects lesion skin regions and then transfers this knowledge to a set of attribute-specific classifiers to detect each particular attribute, which shows better results compared to initialization using weights pre-trained on ImageNet. 
Even though transfer learning is widely used due to its efficiency, it cannot be used when deep learning is not applicable, such as when the number of features is bigger than the number of samples. Therefore, in such cases, combining datasets is more appropriate.} 

For the third direction, some works in combining datasets exist. For example,  in business intelligence, the work reported in \cite{liu2020discovering} tries to customize an analytic dataset at hand with additional features from other datasets by a model that utilizes semantic relationships to deal with the problem. Another example is \cite{rodrigues2020dimensionality}, which presents a transfer learning method for merging brain-computer interface datasets with different dimensionalities for business analytic data. 
Next, \cite{verma2014imputation} proposes a pipeline for imputing and combining genomic datasets, where the merge is conducted based on the set of overlapping SNPs. Note that our work should not be confused with this work since our frameworks use imputation for combining the datasets, while \cite{verma2014imputation} uses imputation as a preprocessing step. Next, \cite{kohnen2009multiple} proposes a method based on multiple imputations for combining two confidential datasets for sharing without revealing confidential information.

Even though multi-task learning methods may help improve each task's learning performance, the sizes between the datasets can be too significant to apply the same type of model. For example, if one dataset has 20 features and 100 samples while the other has 25 features and 1000 samples. Then, linear regression may be suitable for the first one, and a neural network may be more efficient for the second one. This creates challenges in using multi-task learning for these two datasets.

These factors make it even more necessary to combine datasets vertically, and ComImp is suitable for that purpose.

\section{ComImp algorithm for combining datasets}\label{sec-method}

\begin{algorithm}
\caption{\textbf{ComImp algorithm} }\label{alg-comImp}
\hspace*{\algorithmicindent} \textbf{Input:} 
\begin{enumerate}
    \item datasets $\mathcal{D}_i = \{ \mathcal{X}_i, \mathbf{y}_i\}, i = 1,...,r$
    \item $\mathcal{F}_i$: set of features in $\mathcal{X}_i$
    \item $g(\mathcal{X}, \mathcal{H})$: a transformation that rearranges features in $\mathcal{X}$ to follow the order of the features in the set of features $\mathcal{H}$ and insert empty columns if a feature in $\mathcal{H}$ does not exist in $\mathcal{X}$
    \item Imputer $I$.
\end{enumerate}

\hspace*{\algorithmicindent} \textbf{Procedure:} 
\begin{algorithmic}[1]
    \State $\mathcal{F} = \bigcup_{i=1}^r \mathcal{F}_i$
    \State  $\mathcal{X}_i^{*} \leftarrow g(\mathcal{X}_i, \mathcal{F})$
    \State $\mathcal{X}^* = \begin{pmatrix}
        \mathcal{X}_1^*\\
        \mathcal{X}_2^*\\
        \vdots\\
        \mathcal{X}_r^*
    \end{pmatrix}, 
    \mathbf{y} = \begin{pmatrix}
        \mathbf{y}_1\\
        \mathbf{y}_2\\
        \vdots\\
        \mathbf{y}_r
    \end{pmatrix}$
    \State $\mathcal{X} \leftarrow $ imputed version of $\mathcal{X}^*$ using imputer I.
\end{algorithmic}
\textbf{Return:} $\mathcal{D} = \{\mathcal{X}, \mathbf{y}\}$, the combined dataset of $\mathcal{D}_1,..., \mathcal{D}_r$.
\end{algorithm}
 
This section details our ComImp framework for combining datasets. The general ComImp framework is shown in Algorithm \ref{alg-comImp}. Suppose that we have datasets $\mathcal{D}_1 = \{\mathcal{X}_1, \mathbf{y}_1\}$,...,
$\mathcal{D}_r =  \{\mathcal{X}_r, \mathbf{y}_r\}$. 
In addition, assume that the $(p_1+i)^{th}$ column of $\mathcal{D}_1$ and $i^{th}$ column of $\mathcal{D}_2$ refer to the same feature. Let $\mathcal{F}$ be the union of the sets of features in $\mathcal{X}_1, \mathcal{X}_2,..., \mathcal{X}_r$ in an ordered way (step 1 of the algorithm). Then, in step 2 of the algorithm, we rearrange features in each $\mathcal{X}_i$ to follow the order of the features in $\mathcal{F}$  and insert an empty column if a feature in $\mathcal{F}$ does not exist in $\mathcal{X}_i$. This results in $\mathcal{X}_i^*$ for each $i=1,...,r$. Next, in step 3, we stack $\mathcal{X}_i, i=1,..,r$ vertically and do similarly for the corresponding labels to get $\mathcal{X}^*$ and $\mathbf{y}$. Finally, we impute the data $\mathcal{X}^*$ to fill in these entries and create a united, complete dataset.

As an illustration, we use the datasets $\mathcal{D}_1, \mathcal{D}_2$ in Section \ref{sec-intro}. A newly merged dataset $\mathcal{D}$ can be formed by the following process:
\begin{enumerate}
    \item Step 1: The union of set of features $\mathcal{F}$ consists of \textit{height, weight, and calo/meal},
    \item Step 2:
    \begin{align*}
    \small
        \mathcal{X}^*_1&= 
        \begin{pmatrix}
    \text{person} &\text{height} & \text{weight} & \text{(calo/meal)}\\
    1 & 120 & 80 & *\\
    2 & 150 & 70 & *\\
    3 & 140 & 80 & *\\
    4 & 135 & 85 & *
    \end{pmatrix},\\
        \mathcal{X}^*_2&= 
        \begin{pmatrix}
    \text{person} &\text{height} & \text{weight} & \text{calo/meal}\\
    5 & * & 90 & 100 \\
    6 & * & 85 & 150\\
    7 & * & 92 & 170
    \end{pmatrix},
\end{align*}

    \item Stack $\mathcal{X}_1^*, \mathcal{X}_2^*$ vertically, and stack $\mathbf{y}_1,  \mathbf{y}_2$ vertically
    \begin{equation*}
    \small
        \mathcal{X}^*= 
        \begin{pmatrix}
    \text{person} &\text{height} & \text{weight} & \text{calo/meal}\\
    1 & 120 & 80 & *\\
    2 & 150 & 70 & *\\
    3 & 140 & 80 & *\\
    4 & 135 & 85 & *\\
    5 & * & 90 & 100 \\
    6 & * & 85 & 150\\
    7 & * & 92 & 170
    \end{pmatrix},
\end{equation*}
and 
\begin{equation*}
    \mathbf{y}
    = \begin{pmatrix}
        80\\ 90 \\ 85\\ 95\\
        100 \\ 95\\82
    \end{pmatrix}
\end{equation*}
\item $\mathcal{X}=$ imputed version of $\mathcal{X}^*$. Suppose that mean imputation is being used, then
    \begin{equation*}
    \small
        \mathcal{X}=
        \begin{pmatrix}
    \text{(person)} &\text{(height)} & \text{(weight)} & \text{(calo/meal)}\\
    1 & 120 & 80 & 140\\
    2 & 150 & 70 & 140\\
    3 & 140 & 80 & 140\\
    4 & 135 & 85 & 140\\
    5 & 136.25 & 90 & 100 \\
    6 & 136.25 & 85 & 150\\
    7 & 136.25 & 92 & 170
    \end{pmatrix},
\end{equation*}
\item $\mathcal{D}=(\mathcal{X},\mathbf{y}).$
\end{enumerate}

\section{Combining datasets with dimension reduction (PCA-ComImp)}\label{sec-pca-comimp}
In many scenarios, the number of features that is available in one dataset but not available in other datasets may be very large. Therefore, if ComImp is applied directly to those datasets, the imputation process may introduce too many noises and biases into the combined dataset that it may suppress the advantages of combining datasets. Then, it would be preferable to conduct dimension reduction on the non-overlapping features of the datasets to reduce the number of values that must be imputed. This is not only to reduce the bias introduced by imputation to the model but also to speed up the imputation process. To start, we adapt similar notations to \cite{nguyen2022principle}, let $pca(A)$ be a function of a data matrix $A$. 
The function returns $(\mathcal{R}_A, V)$ where $\mathcal{R}_A$ is the PCA-reduced version of $A$, and $V$ is the projection matrix where the $i^{th}$ column of $V$ is the eigenvector corresponding to the $i^{th}$ largest eigenvalue. To simplify notations, we also sometimes use the notation $\mathcal{F}$ that refers to a set of features to indicate a subset of a dataset that contains only these features.
Then, we have Algorithm \ref{alg-PCAcomImp} for combining two datasets with dimension reduction.

The algorithm starts by finding the union, intersection, and complements of feature sets as in step 1 of the procedure. Note that $\mathcal{R}_1 = \mathcal{F}_1\setminus\mathcal{F}_2$ means that $\mathcal{R}_1$ is the set of features that belong to $\mathcal{D}_1$ but not $\mathcal{D}_2$. Similar for $\mathcal{R}_2$. Then, in the second step, the procedure uses PCA to find the reduced version and the projection matrix $V$ of $\mathcal{S}_i^{(tr)}$ (the subset of the training portion of $\mathcal{D}_i$ that consists of only the features $\mathcal{S}_i$). In step 3, we find the set $\mathcal{H}$ of the union of features between the common features of $\mathcal{F}_1, \mathcal{F}_2$, the set of reduced features of the features that are available in $\mathcal{F}_1$ but not in $\mathcal{F}_2$, and the set of reduced features of the features that are available in $\mathcal{F}_2$ but not in $\mathcal{F}_1$. Next, similar to ComImp, we create $\mathcal{X}_i^*$ in step 4, stack them together in step 5, and impute them in step 6.

To combine multiple datasets that have many features not in common, one can operate sequentially. For example, if we have $\mathcal{D}_1,..., \mathcal{D}_r$ then we can first combine $\mathcal{D}_1$ and $\mathcal{D}_2$ to get $\mathcal{D}_{12}$, and then merge $\mathcal{D}_{12}$ with $\mathcal{D}_3$ to get $\mathcal{D}_{123}$ and so on. However, attention should be paid to the number of entries to be imputed. Otherwise, the noise and biases introduced by imputing too many missing values may surpass the benefits of combining information.

\begin{algorithm}
\caption{\textbf{PCA-ComImp} }\label{alg-PCAcomImp}
\hspace*{\algorithmicindent} \textbf{Input:} 
\begin{enumerate}
    \item Datasets $\mathcal{D}_i = \{ \mathcal{X}_i, \mathbf{y}_i\}, i = 1,2$ consists of $\mathcal{D}_i^{(tr)} = \{ \mathcal{X}_i^{(tr)}, \mathbf{y}_i^{(tr)}\}, i = 1,2$ as the training sets and $\mathcal{D}_i^{(ts)} = \{ \mathcal{X}_i^{(ts)}, \mathbf{y}_i^{(ts)}\}$ as the test sets,
    \item $\mathcal{F}_i$: set of features in $\mathcal{X}_i$,
    \item $g(\mathcal{X}, \mathcal{H})$: a transformation that rearranges features in $\mathcal{X}$ to follow the order of the features in the set of features $\mathcal{H}$ and insert empty columns if a feature in $\mathcal{H}$ does not exist in $\mathcal{X}$
    \item Imputer $I$, classifier $\mathcal{C}$.
\end{enumerate}

\hspace*{\algorithmicindent} \textbf{Procedure:} 
\begin{algorithmic}[1]
    \State $\mathcal{F} =  \mathcal{F}_1 \cup \mathcal{F}_2, \; \mathcal{S} = \mathcal{F}_1 \cap \mathcal{F}_2, \; \mathcal{Q}_1 = \mathcal{F}_1\setminus \mathcal{F}_2, \; \mathcal{Q}_2 = \mathcal{F}_2\setminus \mathcal{F}_1$
    \State $(\mathcal{R}_i^{(tr)}, V) \leftarrow pca(\mathcal{Q}_i^{(tr)})$ and  
    $\mathcal{R}_i^{(ts)} \leftarrow \mathcal{Q}_i^{(ts)}V$
    \State $\mathcal{H} = \mathcal{S} \cup  \mathcal{R}_1 \cup \mathcal{R}_2$
    \State $\mathcal{X}_i^* \leftarrow g(\mathcal{S} \cup \mathcal{R}_i, \mathcal{H}), i = 1, 2$
    \State $\mathcal{X}^* = \begin{pmatrix}
    \mathcal{X}_1^*\\
    \mathcal{X}_2^*
    \end{pmatrix}, 
    \mathbf{y} = \begin{pmatrix}
    \mathbf{y}_1\\
    \mathbf{y}_2
    \end{pmatrix}$
    \State $\mathcal{X} \leftarrow $ imputed version of $\mathcal{X}^*$ using imputer I.
\end{algorithmic}

\textbf{Return:} $\mathcal{D} = \{\mathcal{X}, \mathbf{y}\}$, the combined dataset of $\mathcal{D}_1,\mathcal{D}_2$.
\end{algorithm}
\section{Choices of missing data imputation methods}\label{sec:choice} 
Even though there are many imputation methods that one can use, several factors should be considered when choosing the imputer for ComImp, such as speed, imputation quality, and the ability to predict sample by sample.
When the datasets to be combined are small, methods that are slow but can give promising imputation quality, such as kNN, MICE \cite{buuren2010mice}, missForest \cite{stekhoven2012missforest}, can be considered. However, these methods can be too slow for moderate large-size \cite{NGUYEN2022108082}, and therefore cannot be used for high dimensional data.
Next, matrix decomposition methods may not be suitable in many scenarios, such as when the data is not of low rank. In addition, since the imputation is done via matrix factorization, those methods are not suitable for imputation in online learning, which requires the handling of each sample as it comes \cite{nguyen4260235pmf}. Moreover, when the uncertainty of imputed values is of interest, then Bayesian techniques can be used. However, Bayesian imputation techniques can be slow. Recently, DIMV \cite{vu2023conditional} was introduced as a scalable imputation method that estimates the conditional distribution of the missing values based on a subset of observed entries. As a result, the method provides an explainable imputation, along with the confident region, in a simple and straight forward manner.

Note that the choice of imputation methods can affect the performance of the model trained on the merged dataset as well. An example can be seen via the following theorem
\begin{theorem}
Assume that we have two datasets $\mathcal{D}_1 = \{\mathbf{U}, \mathbf{y}\}, \mathcal{D}_2 = \{\mathbf{V}, \boldsymbol{z}\}$ where $\mathbf{U}$ 
, $\mathbf{V}$
are inputs, and $\mathbf{y}, \boldsymbol{z}$ are the labels, such that
\begin{align}
    \mathbf{U} &= (\mathbf{1}_n \; \;\mathbf{u}_1)  = \begin{pmatrix}
        1 & u_{11}\\
        1 & u_{21}\\
        \vdots &\vdots\\
        1 & u_{n1}
    \end{pmatrix}, \\
    \mathbf{V} &= (\mathbf{1}_m \; \;\mathbf{v}_1 \; \;\mathbf{v}_2)  = \begin{pmatrix}
        1 & v_{11} & v_{12}\\
        1 & v_{21} & v_{22} \\
        \vdots &\vdots & \vdots\\
        1 & v_{m1} & v_{m2}
    \end{pmatrix},    
\end{align}
where $u_{ij}, v_{ij}\in \mathbb{R}$, and 
\begin{equation}
    \mathbf{y}=\begin{pmatrix}
    
        y_1\\y_2\\\vdots\\y_n
    \end{pmatrix}, \boldsymbol{z}=\begin{pmatrix}
        z_1\\z_2\\\vdots\\z_m
    \end{pmatrix}, \mathbf{Y} = \begin{pmatrix}
        \mathbf{y}\\\boldsymbol{z} 
    \end{pmatrix}.
\end{equation}

Next, let $\bar{\mathbf{v}}_2$ be the mean of $\mathbf{v}_2$. If mean imputation is being used for ComImp then the resulting combined input is 
\begin{equation}
\mathbf{X} =  \begin{pmatrix}
    1_n & \mathbf{u}_1 & \bar{\mathbf{v}}_2\\
    1_m & \mathbf{v}_1 & \mathbf{v}_2
\end{pmatrix}.
\end{equation}
 Then, we have the following relation between the sum of squared errors ($SSE$) of the model fitted on $\mathcal{D} = \{\mathbf{X}, \mathbf{Y}\}$, 
\begin{equation}\label{equation-inq}
    SSE_{\mathcal{D}} \ge SSE_{\mathcal{D}_1}+ SSE_{\mathcal{D}_2}.
\end{equation}
\end{theorem}
\begin{proof}
Note that equation \eqref{equation-inq} is equivalent to 
\begin{equation}
    \mathbf{Y}'(I-\mathbf{H}_x)\mathbf{Y}\le
    \mathbf{y}'(I-\mathbf{H}_u)\mathbf{y} + \mathbf{z}'(I-\mathbf{H}_v)\mathbf{z},
\end{equation}
where 
\begin{align}
    \mathbf{H}_u &= \mathbf{U}(\mathbf{U}'\mathbf{U})^{-1}\mathbf{U}',\\ 
    \mathbf{H}_v &= \mathbf{V}(\mathbf{V}'\mathbf{V})^{-1}\mathbf{V}', \\
    \mathbf{H}_x &= \mathbf{X}(\mathbf{X}'\mathbf{X})^{-1}\mathbf{X}'.
\end{align}
Hence, we want to prove that 
\begin{equation}
    \mathbf{y}'\mathbf{H}_u \mathbf{y}+\boldsymbol{z}'\mathbf{H}_v \boldsymbol{z}\ge \mathbf{Y}'\mathbf{H}_x\mathbf{Y},
\end{equation}
    Without loss of generality, assume that the data is centered so that $\sum\limits_{i=1}^n u_{i1} = 0$, $\sum\limits_{i=1}^n v_{i1} = 0$, $\sum\limits_{i=1}^n v_{i2} = 0$, so $\bar{\mathbf{v}}_2 = 0$. Next, let $\alpha = \sum\limits_{i=1}^ny_i$, $\alpha_1 = \sum\limits_{i=1}^n y_iu_{i1}$, $a_1=\sum\limits_{i=1}^n u_{i1}^2$, we have
    \vspace{-5pt}
    \begin{align*}
        \mathbf{y}'\mathbf{U} &= \begin{pmatrix}
        \sum\limits_{i=1}^ny_i & \sum\limits_{i=1}^n y_iu_{i1}
        \end{pmatrix} = \begin{pmatrix}
            \alpha & \alpha_1
        \end{pmatrix},\\
        \mathbf{U}'\mathbf{U} &= \begin{pmatrix}
            1 & \sum\limits_{i=1}^n u_{i1}\\
            \sum\limits_{i=1}^n u_{i1} & \sum\limits_{i=1}^n u_{i1}^2
        \end{pmatrix} = \begin{pmatrix}
            n & 0 \\ 0 & a_1
        \end{pmatrix},\\
        \mathbf{U}'\mathbf{y} &= \begin{pmatrix}
            \sum\limits_{i=1}^n y_i \\ \sum\limits_{i=1}^n y_iu_{i1}
        \end{pmatrix} = \begin{pmatrix}
            \alpha \\ \alpha_1
        \end{pmatrix}.
    \end{align*}
    This implies 
    \begin{align}\label{yHuy}
        \mathbf{y}'\mathbf{H}_u \mathbf{y} &=  \mathbf{y}'\left(\mathbf{U}(\mathbf{U}'\mathbf{U})^{-1}\mathbf{U}'\right) \mathbf{y}
        = \left(\mathbf{y}'\mathbf{U}\right)\left(\mathbf{U}'\mathbf{U}\right)^{-1} \left(\mathbf{U}'\mathbf{y}\right) \notag\\
        &= \frac{\alpha^2}{n} + \frac{\alpha_1^2}{a_1}.
    \end{align}
Similarly, let $\beta = \sum\limits_{i=1}^m z_i$, $\beta_1 = \sum\limits_{i=1}^m z_iv_{i1}$, $\beta_2 = \sum\limits_{i=1}^m z_iv_{i2}$, $b_1=\sum\limits_{i=1}^m v_{i1}^2$,
$b_1=\sum\limits_{i=1}^m v_{i2}^2$,
$c=\sum\limits_{i=1}^m v_{i1}v_{i2}$. Then, 
    \begin{align*}
        \mathbf{z}'\mathbf{V} &= \begin{pmatrix}
            \sum\limits_{i=1}^m z_i & \sum\limits_{i=1}^m z_iv_{i1} & \sum\limits_{i=1}^mz_iv_{i12}
        \end{pmatrix}\\
        &= \begin{pmatrix}
            \beta & \beta_1 & \beta_2
        \end{pmatrix},\notag\\
        \mathbf{V}'\mathbf{V} &= \begin{pmatrix}
            m & \sum\limits_{i=1}^m v_{i1} &\sum\limits_{i=1}^m v_{i2}\\
            \sum\limits_{i=1}^m v_{i1} & \sum\limits_{i=1}^m v_{i1}^2 & \sum\limits_{i=1}^m v_{i1}v_{i2}\\
            \sum\limits_{i=1}^m v_{i2} &\sum\limits_{i=1}^m v_{i1}v_{i2} &\sum\limits_{i=1}^m v_{i2}^2
        \end{pmatrix}\\
        &= \begin{pmatrix}
            m & 0 & 0 \\
            0 & b_1 & c\\
            0 & c & b_2
        \end{pmatrix},\notag\\
\mathbf{V}'\mathbf{z} &= \begin{pmatrix}
            \sum\limits_{i=1}^m z_i \\\sum\limits_{i=1}^m z_iv_{i1} \\ z_iv_{i12}
        \end{pmatrix} = \begin{pmatrix}
            \beta \\ \beta_1 \\ \beta_2
        \end{pmatrix}.
            \end{align*}
Therefore,
    \begin{align}\label{zHzz}
        \mathbf{z}'\mathbf{H}_z \mathbf{z} &=  \mathbf{z}'\left(\mathbf{V}(\mathbf{V}'\mathbf{V})^{-1}\mathbf{V}'\right) \mathbf{z}
        = \left(\mathbf{z}'\mathbf{V}\right)\left(\mathbf{V}'\mathbf{V}\right)^{-1} \left(\mathbf{V}'\mathbf{z}\right) \notag\\
        &= \frac{\beta^2}{m} + \frac{\beta_1^2b_2+\beta_2^2b_1-2\beta_1\beta_2c}{b_1b_2-c^2}\notag\\
        &= \frac{\beta^2}{m} + \frac{\left(\beta_1b_2-\beta_2c\right)^2}{b_2\left(b_1b_2-c^2\right)} + \frac{\beta_2^2}{b_2}.
    \end{align}
    \vspace{-5pt}
    Besides,
    \begin{align*}
        \mathbf{Y}'\mathbf{X}&= \begin{pmatrix}
            \sum\limits_{i=1}^n y_i +\sum\limits_{i=1}^m z_i & \sum\limits_{i=1}^n y_iu_{i1} +\sum\limits_{i=1}^m z_iv_{i1} & \sum\limits_{i=1}^mz_iv_{i12}
        \end{pmatrix}\\
        &= \begin{pmatrix}
            \alpha + \beta & \alpha_1 + \beta_1 & \beta_2
        \end{pmatrix}, \notag\\
        \mathbf{X}'\mathbf{X} &= \begin{pmatrix}
            m & \sum\limits_{i=1}^n u_{i1} +\sum\limits_{i=1}^m v_{i1} &\sum\limits_{i=1}^m v_{i2}\\
            \sum\limits_{i=1}^n u_{i1} +\sum\limits_{i=1}^m v_{i1} & \sum\limits_{i=1}^n u_{i1}^2 + \sum\limits_{i=1}^m v_{i1}^2 & \sum\limits_{i=1}^m v_{i1}v_{i2}\\
            \sum\limits_{i=1}^m v_{i2} &\sum\limits_{i=1}^m v_{i1}v_{i2} &\sum\limits_{i=1}^m v_{i2}^2
        \end{pmatrix}\\
        &= \begin{pmatrix}
            m & 0 & 0 \\
            0 & a_1 + b_1 & c\\
            0 & c & b_2
        \end{pmatrix},\\\notag
        \mathbf{X}'\mathbf{Y}&= \begin{pmatrix}
            \sum\limits_{i=1}^n y_i +\sum\limits_{i=1}^m z_i \\ \sum\limits_{i=1}^n y_iu_{i1} +\sum\limits_{i=1}^m z_iv_{i1} \\ \sum\limits_{i=1}^mz_iv_{i12}
        \end{pmatrix}= \begin{pmatrix}
            \alpha + \beta \\ \alpha_1 + \beta_1 \\ \beta_2
        \end{pmatrix}. 
    \end{align*}
Hence, 
\begin{align}\label{YHxY}
        \mathbf{Y}'\mathbf{H}_x \mathbf{Y} &=  \mathbf{Y}'\left(\mathbf{X}(\mathbf{X}'\mathbf{X})^{-1}\mathbf{X}'\right) \mathbf{Y}
        = \left(\mathbf{Y}'\mathbf{X}\right)\left(\mathbf{X}'\mathbf{X}\right)^{-1} \left(\mathbf{X}'\mathbf{Y}\right)\notag\\
        &= \frac{\left(\alpha + \beta\right)^2}{m} + \frac{\left[\alpha_1b_2+\beta_1b_2-\beta_2c\right]^2}{b_2\left[a_1b_2+b_1b_2-c^2\right]} + \frac{\beta_2^2}{b_2}.
    \end{align}
Now, applying Cauchy-Schwarz inequality, we have
\begin{equation}
        \frac{\alpha^2}{n} +\frac{\beta^2}{m}\ge \frac{\left(\alpha+\beta\right)^2}{m+n},\label{CS1}
    \end{equation}
\begin{align}\label{CS2}
        \frac{\alpha_1^2}{a_1} + \frac{\left(\beta_1b_2-\beta_2c\right)^2}{\beta_2\left(b_1b_2-c^2\right)} &= \frac{1}{b_2}\left[\frac{\alpha_1^2b_2^2}{a_1b_2} +\frac{\left(\beta_1b_2-\beta_2c\right)^2}{b_1b_2-c^2}\right]\notag\\
        &\ge \frac{\left(\alpha_1b_2+\beta_1b_2-\beta_2c\right)^2}{b_2\left(a_1b_2+b_1b_2-c^2\right)}.
    \end{align}
From \eqref{yHuy}, \eqref{zHzz}, \eqref{YHxY}, \eqref{CS1} and \eqref{CS2}, we have
    $\mathbf{y}'\mathbf{H}_u \mathbf{y}+\boldsymbol{z}'\mathbf{H}_v \boldsymbol{z}\ge \mathbf{Y}'\mathbf{H}_x \mathbf{Y},$ as desired.
\end{proof}
From the theorem, we can see that in the given setting stated in the theorem, the SSE of the regression model fitted on the merged data is bigger than the sum of the SSE of the models trained on each component dataset. Hence, using mean imputation in this case actually leads to worse performance.  

\section{Experiments}\label{sec-exper}
In this section, we present various experiments on a variety of scenarios to explore the potential of the proposed techniques. Throughout all the following experiments, if not further commented, we use soft-impute \cite{mazumder2010spectral} as the imputation method, each experiment is repeated 10000 times before reporting the mean and variance of the results, and the train and test set are split with ratio 1:1. The descriptions of the datasets used are as in Table \ref{table_info_datasets}. The codes for the experiments are available at \href{https://github.com/thunguyen177/ComImp}{github.com/thunguyen177/ComImp}.

\begin{table}[htbp]
\caption{Descriptions of datasets used in the experiments}
\begin{center}
    \begin{tabular}{|c|c|c|c|}
		\hline
		Dataset & \# Classes & \# Features & \# Samples \\
		\hline
		Seed \cite{charytanowicz2010complete} & 3 & 7 & 210\\\hline
		Wine & 3 & 13 & 178 \\\hline
		Epileptic Seizure (\cite{eeg_seizure}) & 2 &  178 &  11,500\\\hline
		Gene & $5$ & $20531$ & $801$ \\\hline
	\end{tabular}
	\label{table_info_datasets}
\end{center}
\end{table}
 \subsection{Regression simulation}\label{sec-regression-simul}
\begin{table}[htbp]
\caption{mean $\pm$ variance of MSE on test sets of regression models fitted on $\mathcal{D}_1,\mathcal{D}_2,\mathcal{D}$, respectively.}
\label{tab-regr}
\begin{center}
\begin{tabular}{|c|c|c|c|}
\hline
model & \bf $f_1$  &{\bf $f_2$}&{\bf $f$}\\ \hline 
MSE & 3.327 $\pm$ 0.093 & 3.044 $\pm$ 0.131 & 0.734 $\pm$ 0.002\\\hline
\end{tabular}
\end{center}
\end{table}

 \paragraph{Data generation} We run a Monte Carlo experiment for regression with 10000 repeats. For each loop, we generate a dataset $\mathcal{D}_1$ of 300 samples, and a dataset $\mathcal{D}_2$ of 200 samples
 based on the following relation
 \begin{equation}
     Y = 1 + X_1+0.5X_2+X_3+\epsilon 
 \end{equation}
where $\mathbf{X} = (X_1, X_2, X_3)^T$
follows a multivariate Gaussian distribution with the following mean and covariance matrix
\begin{equation}
    \boldsymbol {\mu }=\begin{pmatrix}
        1\\2\\0.5   
    \end{pmatrix}, 
    \mathbf{\Sigma} = \begin{pmatrix}
        1 & 0.5 & 0.3\\
        0.5 & 1 & 0.4\\
        0.3 & 0.4 & 1
    \end{pmatrix} 
\end{equation}
and $\epsilon \sim \mathcal{ N}(0,0.2\,I)$, where $I$ is the identity matrix.

  Then, we delete the first feature of $\mathcal{D}_1$ and the second feature of $\mathcal{D}_2$. To make it more similar to practical settings, where there the data collecting machines/sensors can also be the source of variation, we added Gaussian noise with variance $0.05$ to the second feature and Gaussian noise with variance $0.1$ to the third feature of $\mathcal{D}_2$.  Then, $\mathcal{D}_1, \mathcal{D}_2$ is split again into training and testing sets ratio 7:3, respectively. 
  
  \textbf{Experimental setup and results. } For each simulation, after generating the data, we use ComImp to merge the training sets of $\mathcal{D}_1$ and $\mathcal{D}_2$ into the corresponding training set of $\mathcal{D}$. We do similarly for the test sets. Next, we fit a regression model on each training set of $\mathcal{D}_1, \mathcal{D}_2, \mathcal{D}$, which gives fitted models $f_1,f_2,f$, respectively. We record the MSE of $f_1, f_2, f_3$ on the corresponding test set of $\mathcal{D}_1, \mathcal{D}_2, \mathcal{D}$, respectively. Note that the test set of $\mathcal{D}$ is the merge between the test sets of $\mathcal{D}_1$ and $\mathcal{D}_2$. Then, we report the mean and variance of the mean square error (MSE) on test sets of 10000 runs in Table \ref{tab-regr}.
  
  \textbf{Analysis. } From Table \ref{tab-regr}, we see that the MSE of $f$ (the model fitted on merged training parts of $\mathcal{D}_1, \mathcal{D}_2$) on the test of $\mathcal{D}$ is significantly lower than $f_1$ and $f_2$. Also, recall that the test set of $\mathcal{D}$ is the merge between the test sets of $\mathcal{D}_1$ and $\mathcal{D}_2$, and the labels are stacked accordingly. This shows the power of merging datasets via ComImp for regression.

\subsection{Experiments of merging two datasets}\label{sec-seed-result}

\paragraph{Data} We conduct experiments on the Seed dataset.
We delete the first two columns of the input in $\mathcal{D}_1$ and split it into training and testing sets of equal sizes. In addition, we delete the last columns of the input data in $\mathcal{D}_2$ and split it into the training and testing sets of equal sizes. We fit models $f_1, f_2$ on the training set of $\mathcal{D}_1, \mathcal{D}_2$, respectively. In addition, we use ComImp to combine the training sets of $\mathcal{D}_1, \mathcal{D}_2$ to the training sets of $\mathcal{D}$, and do similarly for the testing portions. Then, we fit a model $f$ on the training set of $\mathcal{D}$. We repeat the experiment 1000 times and report the mean and variance of the accuracy in Table \ref{tab-merge2}. 

\paragraph{Analysis} Note that in Table \ref{tab-merge2}, for $f$, the third (fifth) column contains the results on the portion of the test set of $\mathcal{D}$ that corresponds to the test set of $\mathcal{D}_1 (\mathcal{D}_2)$. From Table \ref{tab-merge2}, we can see that the performance of the model fitted on the merged data improves significantly for $\mathcal{D}_2$, the smaller dataset, and slightly for $\mathcal{D}_1$. Specifically, when using Logistic Regression, the accuracy of $f$ on $\mathcal{D}_2$ is 0.896, which is 3.8\% higher than $f_2$, whose accuracy is 0.858. Therefore, ComImp benefits the small sample dataset much better than the large sample one.

\begin{table*}[ht]
\caption{mean $\pm$ variance of the accuracy on test sets of  models fitted on $\mathcal{D}_1,\mathcal{D}_2,\mathcal{D}$ of the Seed dataset, respectively.}
\label{tab-merge2}
\begin{center}
\begin{tabular}{|c|c|c|c|c|}
\hline
\multicolumn{1}{|c|}{}& \multicolumn{2}{|c|}{\bf test set of $\mathcal{D}_1$}  &\multicolumn{2}{|c|}{\bf test set of $\mathcal{D}_2$}
\\ \hline 
\multicolumn{1}{|c|}{Classifier}& \multicolumn{1}{|c|}{\bf ${f}_1$}  &\multicolumn{1}{|c|}{\bf ${f}$} &\multicolumn{1}{|c|}{\bf ${f}_2$} &\multicolumn{1}{|c|}{\bf ${f}$}
\\ \hline 
Logistic Regression & 0.909 $\pm$ 0.001 &0.913 $\pm$ 0.001 & 0.858 $\pm$ 0.006 & 0.896 $\pm$ 0.003\\\hline
SVM & 0.917 $\pm$ 0.001 & 0.925 $\pm$ 0.001 & 0.882 $\pm$ 0.003 & 0.890 $\pm$ 0.003\\\hline
\end{tabular}
\end{center}
\end{table*}

\subsection{Experiments of merging three datasets}
\paragraph{Data} We conduct experiments on the Wine dataset.
We first randomly split the data into three datasets $\mathcal{D}_1, \mathcal{D}_2, \mathcal{D}_3$. Then, we delete the first column of the input of $\mathcal{D}_1$, the last 8 columns of $\mathcal{D}_2$, and the $5^{th}$ and $6^{th}$ column of the input of $\mathcal{D}_3$. 

\paragraph{Analysis} Note that in Table \ref{tab-merge3} of the results for this experiment, for $f$, the third/fifth/seventh column contains the results on the portion of the test set of $\mathcal{D}$ that corresponding to the test set of $\mathcal{D}_1/ \mathcal{D}_2/ \mathcal{D}_3$. From Table \ref{tab-merge3}, we can see that the performance of the model fitted on the merged data improves the accuracy on the test sets of all the component datasets $\mathcal{D}_1, \mathcal{D}_2, \mathcal{D}_3$. Note that the sizes of $\mathcal{D}_1, \mathcal{D}_2, \mathcal{D}_3$ are all relatively small. So, we can see that ComImp improves the model fitting process when combining small datasets together.

\begin{table*}[ht!]
\caption{mean $\pm$ variance of the accuracy on test sets of models fitted on $\mathcal{D}_1,\mathcal{D}_2, \mathcal{D}_3, \mathcal{D}$ of the Wine dataset, respectively.}
\label{tab-merge3}
\begin{center}
\small
\begin{tabular}{|c|c|c|c|c|c|c|}
\hline
\multicolumn{1}{|c|}{}& \multicolumn{2}{|c|}{\bf test set of $\mathcal{D}_1$}  &\multicolumn{2}{|c|}{\bf test set of $\mathcal{D}_2$}&\multicolumn{2}{|c|}{\bf test set of $\mathcal{D}_3$}\\ \hline 
\multicolumn{1}{|c|}{Classifier}& \multicolumn{1}{c}{\bf ${f}_1$}  &\multicolumn{1}{|c|}{\bf ${f}$} &\multicolumn{1}{|c|}{\bf ${f}_2$} &\multicolumn{1}{|c|}{\bf ${f}$}&\multicolumn{1}{|c|}{\bf ${f}_3$} &\multicolumn{1}{|c|}{\bf ${f}$}
\\ \hline 
Logistic Regression  & 0.912 $\pm$ 0.004 & 0.941 $\pm$ 0.002  & 0.735 $\pm$ 0.010& 0.781 $\pm$ 0.007&0.946 $\pm$ 0.003& 0.971 $\pm$ 0.001\\\hline
SVM & 0.934 $\pm$ 0.003 & 0.955 $\pm$ 0.001 & 0.793 $\pm$ 0.007& 0.826 $\pm$ 0.005 &0.953 $\pm$ 0.002&0.968 $\pm$ 0.001\\\hline
\end{tabular}
\end{center}
\end{table*}

\subsection{ComImp with transfer learning}

\paragraph{Data} We conduct experiments on multi-variate time series datasets related to EEG signals. The \ac{EEG} dataset is collected for epileptic seizure recognition \cite{eeg_seizure}
. The dataset used for the experiment consists of 178 time points equivalent to one second of \ac{EEG} recording. The recordings of seizure activity and non-seizure activity are labeled in the dataset. We first randomly split the data into two datasets $\mathcal{D}_1, \mathcal{D}_2$ with 6:4 ratio. Then,  to simulate the bad EEG recordings with missing values, we delete the first 16 columns of $\mathcal{D}_1$ and the last 17 columns of $\mathcal{D}_2$. Next, we split $\mathcal{D}_1, \mathcal{D}_2$ into training and testing sets with a 7:3 ratio. 
To compare our methods with transfer learning, we conduct the following two procedures: 

\begin{itemize}
\item We train a fully connected two-layered neural network ($f_1$) on the training portion of $\mathcal{D}_1$. 
Next, we transfer the weights of $f_1$ to train a fully connected network $f_2$ on the training portion of $\mathcal{D}_2$ and fine-tune $f_2$. 

\item For our method, we use ComImp to combine $\mathcal{D}_1$ and $\mathcal{D}_2$, which gives us $\mathcal{D}$ with $\mathcal{D}_{train}$ is the merge between the training portion of $\mathcal{D}_1$ and $\mathcal{D}_2$. In addition, $\mathcal{D}_{test_1}$ corresponding to the testing portion of $\mathcal{D}_1$, and $\mathcal{D}_{test_2}$ corresponding to the testing portion of $\mathcal{D}_2$. We train a model $f$ on $\mathcal{D}_{train}$.  Then, we transfer the weights of $f$ onto the training portion of $\mathcal{D}_1$ and fine-tune the model, which gives us model $f^{ComImp}_{1}$. We do similarly for $\mathcal{D}_2$, which gives us $f^{ComImp}_{2}$.
\end{itemize}

We {run the model for 10,000 epochs} and report the classification accuracy of $f_1$ on the testing portion of $\mathcal{D}_1$, and the classification accuracy of $f_2$ on the testing portion of $\mathcal{D}_2$ in Table \ref{tab-eeg}. In addition, we report the classification accuracy of $f^{ComImp}_{1}, f^{ComImp}_{2}$ on $\mathcal{D}_{test_1}$, $\mathcal{D}_{test_2}$, which corresponding to the testing portion of $\mathcal{D}_1, \mathcal{D}_2$, in column 3 and 5 of Table \ref{tab-eeg}, respectively.

\textbf{Analysis.} We observe in Table \ref{tab-eeg} that using ComImp to combine the training parts of $\mathcal{D}_1, \mathcal{D}_2$, train a model $f$ on the merged and fine-tuning on the corresponding training parts of $\mathcal{D}_1$ and $\mathcal{D}_2$ yields a better classification accuracy result. Specifically, the performance gain is 2.2 $\%$ for the test of $\mathcal{D}_1 $ and 1.0 $\%$ over the test set of $\mathcal{D}_2$. The gain in accuracy result demonstrates the positive impact of merging datasets by ComImp before fine-tuning.
\subsection{Data imputation performance for missing datasets}
\paragraph{Data} We conduct experiments on the Wine dataset. 
We delete the first two columns of the input in $\mathcal{D}_1$ and split it into training and testing sets of equal sizes. In addition, we delete the last two columns of the input data in $\mathcal{D}_2$. 
Then, we generate missing data at missing rates of $20\%, 40\%, 60\%$ on each training/testing set. In addition, we use ComImp to combine the corresponding training and testing sets of $\mathcal{D}_1, \mathcal{D}_2$ to the training and testing sets of $\mathcal{D}$. 

For the ComImp approach, the missing values are automatically filled after merging the datasets. For the non-ComImp approach, we use softImpute to impute missing values. Then, we fit a models $f_1, f_2, f$ on the training set of $\mathcal{D}_1, \mathcal{D}_2, \mathcal{D}$, respectively. 

\paragraph{Analysis.} Note that in the table of results (Table \ref{tab-imputation}), for $f$, the fourth (sixth) column contains the results on the portion of the test set of $\mathcal{D}$ that corresponding to the test set of $\mathcal{D}_1 (\mathcal{D}_2)$. From Table \ref{tab-imputation}, we can see that the performance of $f$ on the test set of $\mathcal{D}_1$ is the same as $f_1$ in most cases. However, the performance of $f$ on the test set of $\mathcal{D}_2$ is significantly better than $f_2$. For example, at 80\% missing rate, when using SVM as the classifier, the accuracy of $f$ on the test set of $\mathcal{D}_2$ is $40.692$, which is 8.9\% higher than $f_2$, whose accuracy is 0.603. Note that $\mathcal{D}_1$ has 125 samples and $\mathcal{D}_2$ has 53 samples. Therefore, combining datasets aid the dataset with a smaller sample size significantly.
\begin{table*}[ht!]
\caption{mean $\pm$ variance of the accuracy on test sets of  models fitted on $\mathcal{D}_1,\mathcal{D}_2$ of the Wine dataset, respectively.}
\label{tab-imputation}
\begin{center}
\small
\begin{tabular}{|c|c|c|c|c|c|}

\hline
\multicolumn{2}{|c|}{}& \multicolumn{2}{|c|}{\bf test set of $\mathcal{D}_1$}  &\multicolumn{2}{|c|}{\bf test set of $\mathcal{D}_2$}
\\ \hline 
\multicolumn{1}{|c|}{missing rate}&\multicolumn{1}{|c|}{Classifier}& \multicolumn{1}{|c|}{\bf ${f}_1$}  &\multicolumn{1}{|c|}{\bf ${f}$} &\multicolumn{1}{|c|}{\bf ${f}_2$} &\multicolumn{1}{|c|}{\bf ${f}$}
\\ \hline 
\multirow{2}{*}{20\%}&Logistic Regression  & 0.917 $\pm$ 0.001 & 0.908 $\pm$ 0.001  & 0.859 $\pm$ 0.007& 0.874 $\pm$ 0.004\\ \cline{2-6}
&SVM & 0.939 $\pm$ 0.001 & 0.939 $\pm$ 0.001 & 0.882 $\pm$ 0.006& 0.893 $\pm$ 0.004\\\hline
\multirow{2}{*}{40\%}&Logistic Regression  & 0.883 $\pm$ 0.002 & 0.877 $\pm$ 0.002 & 0.804 $\pm$ 0.009& 0.837 $\pm$ 0.005\\ \cline{2-6}
&SVM & 0.901 $\pm$ 0.001 & 0.901 $\pm$ 0.001 & 0.814 $\pm$ 0.009& 0.846 $\pm$ 0.005\\ \hline
\multirow{2}{*}{60\%}&Logistic Regression  & 0.834 $\pm$ 0.002 & 0.831 $\pm$ 0.002 & 0.733 $\pm$ 0.010& 0.783 $\pm$ 0.007\\ \cline{2-6}
&SVM & 0.843 $\pm$ 0.002 & 0.846 $\pm$ 0.002 & 0.716 $\pm$ 0.012 & 0.779 $\pm$ 0.008\\\hline
\multirow{2}{*}{80\%}&Logistic Regression  & 0.762 $\pm$ 0.004 & 0.762 $\pm$ 0.003 & 0.645 $\pm$ 0.012& 0.712 $\pm$ 0.008\\ \cline{2-6}
&SVM & 0.762 $\pm$ 0.004 & 0.766 $\pm$ 0.003 & 0.603 $\pm$ 0.014& 0.692 $\pm$ 0.010\\\hline
\end{tabular}
\end{center}
\end{table*}

\subsection{Combining datasets using dimension reduction}
\paragraph{Data} We conduct experiments on the Gene dataset \cite{Dua:2019}.  First, we split it into $\mathcal{D}_1, \mathcal{D}_2$ with a ratio $7:3$, and then split each of them into halves for training and testing. Then, we delete the first 10,000 columns of the input in $\mathcal{D}_1$ and the last 10,000 columns of the input in $\mathcal{D}_2$. We apply ComImp with PCA and train a neural network $f$ on the merged data. Also, we train a separate model $f_1$ on $\mathcal{D}_1$, and $f_2$ on $\mathcal{D}_2$. We repeat the experiment 100 times and report the mean and variance of the accuracy in Table  \ref{tab-gene}. 

\paragraph{Analysis} Similar to the previous experiments, in Table \ref{tab-gene}, for $f$, the third (fifth) column contains the results on the portion of the test set of $\mathcal{D}$ that corresponding to the test set of $\mathcal{D}_1 (\mathcal{D}_2)$. From Table \ref{tab-gene}, we can see that the PCA-ComImp improves the performance for SVM on the test sets of both $\mathcal{D}_1$ and $\mathcal{D}_2$. For example, on the test set of $\mathcal{D}_2$, $f$ has an accuracy of $0.971$, which is higher than $f_2$ (0.955).

\begin{table*}[ht!]
\caption{mean $\pm$ variance of the accuracy on test sets of regression models fitted on $\mathcal{D}_1,\mathcal{D}_2,\mathcal{D}$ of the Gene dataset, respectively.}
\label{tab-gene}
\begin{center}
\begin{tabular}{|c|c|c|c|c|}
\hline
\multicolumn{1}{|c|}{}& \multicolumn{2}{|c|}{\bf test set of $\mathcal{D}_1$}  &\multicolumn{2}{|c|}{\bf test set of $\mathcal{D}_2$}
\\ \hline 
\multicolumn{1}{|c|}{Classifier}& \multicolumn{1}{|c|}{\bf ${f}_1$}  &\multicolumn{1}{|c|}{\bf ${f}$} &\multicolumn{1}{|c|}{\bf ${f}_2$} &\multicolumn{1}{|c|}{\bf ${f}$}
\\ \hline 
Logistic Regression & 0.998 $\pm$ 0.002 &0.998 $\pm$ 0.002 & 0.996 $\pm$ 0.006 & 0.996 $\pm$ 0.006\\\hline
SVM & 0.995 $\pm$ 0.007 & 0.997 $\pm$ 0.003 & 0.955 $\pm$ 0.021 & 0.971 $\pm$ 0.012\\ \hline
\end{tabular}
\end{center}
\end{table*}

\begin{table}[ht]
\caption{Comparison of transfer learning and transfer learning with ComImp on the \ac{EEG} dataset.}
\label{tab-eeg}
\begin{center}
\begin{tabular}{|c|c|c|c|c|}
\hline
\multicolumn{1}{|c|}{}& \multicolumn{2}{|c|}{\bf test set of $\mathcal{D}_1$}  &\multicolumn{2}{|c|}{\bf test set of $\mathcal{D}_2$}
\\ \hline  
\multicolumn{1}{|c|}{Classifier}& \multicolumn{1}{|c|}{\bf ${f}_1$}  &\multicolumn{1}{|c|}{\bf $f^{ComImp}_{1}$} &\multicolumn{1}{|c|}{\bf ${f}_2$} &\multicolumn{1}{|c|}{\bf $f^{ComImp}_{2}$}
\\ \hline 
Transfer Learning & 0.774 & 0.796 & 0.839  & 0.849 \\
\hline 
\end{tabular}
\end{center}
\end{table}
\subsection{A case analysis of when OsImp may fail}\label{sec-osimp-fail}

This experiment is conducted on the Seed dataset with the same setup as in section \ref{sec-seed-result} except that we delete the first three features of $\mathcal{D}_1$ and the last four features of $\mathcal{D}_2$. The results are reported in Table \ref{tab-osimp-fail}.

From Table \ref{tab-osimp-fail}, we see that the performance of the model trained only on the training set of $\mathcal{D}_1$ and the model trained only on the training set of $\mathcal{D}_2$ are better than the model trained on the merge of $\mathcal{D}_1$ and $\mathcal{D}_2$. Next, note that the Seed dataset has only seven features, and therefore, $\mathcal{D}_1$ and $\mathcal{D}_2$ have only one overlapping feature. This is an example of when the imputation of too many features while having only a few overlapping features may introduce too much noise and bias to compensate for the benefits of having more samples. 

\begin{table*}[ht]
\caption{mean $\pm$ variance of the accuracy on test sets of models fitted on $\mathcal{D}_1,\mathcal{D}_2,\mathcal{D}$ of the Seed dataset, respectively.}
\label{tab-osimp-fail}
\begin{center}
\begin{tabular}{|c|c|c|c|c|}
\hline
\multicolumn{1}{|c|}{}& \multicolumn{2}{|c|}{\bf test set of $\mathcal{D}_1$}  &\multicolumn{2}{|c|}{\bf test set of $\mathcal{D}_2$}
\\ \hline 
\multicolumn{1}{|c|}{Classifier}& \multicolumn{1}{|c|}{\bf ${f}_1$}  &\multicolumn{1}{|c|}{\bf ${f}$} &\multicolumn{1}{|c|}{\bf ${f}_2$} &\multicolumn{1}{|c|}{\bf ${f}$}
\\ \hline 
Logistic Regression & 0.912 $\pm$ 0.001 &0.911 $\pm$ 0.001 & 0.730 $\pm$ 0.017 & 0.675 $\pm$ 0.020\\\hline
SVM & 0.926 $\pm$ 0.001 & 0.923 $\pm$ 0.001 & 0.827 $\pm$ 0.008 & 0.826 $\pm$ 0.010\\\hline
\end{tabular}
\end{center}
\end{table*}

\section{Conclusions} \label{sec-concl}
In this paper, we introduced a novel method, \textbf{ComImp}, to combine datasets. The proposed method increases the sample size and makes use of information from multiple datasets, and is highly beneficial for small sample datasets, as illustrated in the experiments. We also illustrated that ComImp could improve model fitting even more when being used with transfer learning.
In addition, we introduced \textbf{PCA-ComImp}, a variant of ComImp, for combining datasets with a large number of features not in common. 

It is worth noting that imputation may also introduce noise and bias. 
As illustrated in section \ref{sec-osimp-fail}, if the number of overlapping features is too small compared to the number of non-overlapping features, the advantage of having more samples may not be able to compensate for that. In addition, the choice of imputation methods has crucial effects on classification results. 
Therefore, it is also worth investigating more imputation methods for ComImp and PCA-ComImp in the future. Moreover, when the datasets under combination are of high dimension, it is possible to increase the imputation speed of an imputation algorithm by using the PCAI framework \cite{nguyen2022principle}. 

In addition, it is also interesting to see if the idea in this paper can be applied to images or other types of data. Further, the method may also be applicable for combining the datasets for MTL approaches. For example, ComImp can be used in the multi-exit architecture proposed for MTL in \cite{lee2021multi}, where the different datasets need to be combined. 

Next, note that about the proposed PCA-ComImp technique uses PCA, which can lead to a loss of interpretability of the features that are not shared among datasets. Therefore, in the future, it is worth studying how feature selection techniques \cite{nguyen2019faster} can be used instead of dimension reduction for this problem or parallel features selection techniques \cite{nguyen2022parallel} for large datasets. 

Lastly, it is worth investigating further the use of ComImp with domain adaptation for further boosting of performance in medical tasks such as in diabetic retinopathy grading \cite{nguyen2021self}, disease detection \cite{hicks2018deep}, spermatozoa tracking \cite{thambawita2022visem, thambawita2022medico}, life-logging \cite{riegler2023scopesense}, etc. Another interesting problem is how the proposed techniques can be used with data augmentation or synthetic data generation \cite{thambawita2022singan} to ameliorate the deficiency of data problems in medicine and various other fields. These questions are worth investigating and will be the topics for our future research. 

\bibliographystyle{plain}
\bibliography{bibfile, merge_data}

\end{document}